\documentclass[conference]{IEEEtran}

\usepackage{amsmath, amsthm, amsfonts,amssymb,euscript, graphicx,epsfig,enumerate,float,afterpage, subfigure, ifthen, moreverb, algpseudocode,algorithm}
\usepackage{url}
\usepackage{hyperref}
\usepackage{tikz}
\usepackage{placeins}
\usepackage{subfigure}
\usetikzlibrary{arrows.meta, positioning, calc}
\usepackage{xcolor}

\newtheorem{thm}{Theorem}

\newtheorem{lem}{Lemma}

\newtheorem{assumption}{Assumption}

\newcommand{\norm}[1]{||{#1}||}

\newcommand{\curl}[1]{\left\{#1\right\}}
\newcommand{\vect}[1]{\mathbf{#1}}
\newcommand{\parn}[1]{\left(#1\right)}

\newcommand{\sqr}[1]{\left[#1\right]}
\newcommand{\bars}[1]{\left|#1\right|}

\begin{document}

\title{\huge{Epidemiological Causal Graph Identification: Challenges, Identifiability and Algorithms}}

\author{Sambit~Mishra$^*$, Yingying~Wang$^\dagger$, Christine~K.~Johnson$^\dagger$, and Urbashi~Mitra$^*$\\
$^*$ University of Southern California, $^\dagger$ University of California, Davis
\thanks{This work has been funded by one or all of the following grants: ARO W911NF1910269, ARO W911NF2410094, ONR N00014-22-1-2363, NSF CIF-2311653, NSF CIF-2148313, NSF RINGS-2148313, NSF DBI-2412522, and is also supported in part by funds from federal agencies and industry partners as specified in the RINGS program. }}

\maketitle

\begin{abstract}
Causal discovery from observational data is fundamental to statistics and machine learning, yet determining causal direction without interventions necessitates structural assumptions. Existing identifiability research primarily focuses on continuous variables under additive noise models, often neglecting mixed datasets containing ordinal scales, counts, and continuous measurements. This paper investigates causal discovery in Directed Acyclic Graphs (DAGs) where nodes follow either an ordinal distribution (via an ordered logit model) or a regular one-parameter exponential family distribution. We prove that the edge direction between an ordinal and an exponential family node is distributionally identifiable for generic parameter values. Our findings generalize previous Ordinal-Poisson results to the broader exponential family. Computationally, we introduce a score-based exhaustive search and a masked continuous optimization framework using DAGMA for larger graphs. Numerical results validate the theory, recovering edge orientations within a Markov equivalence class that are unidentifiable under classical structural equation models.
\end{abstract}

\IEEEpeerreviewmaketitle

\section{Introduction}

Causal reasoning distinguishes the variables that produce an outcome from those that merely correlate with it. Causal directed acyclic graphs (DAGs) \cite{pearl2009causality} compactly encode these relationships, using nodes for random variables and directed edges for causal influence. We focus on discovery from observations, where the key challenge is that multiple DAGs can induce the same joint distribution, forming a Markov equivalence class (MEC) \cite{peters2017elements}. Constraint-based algorithms such as PC \cite{spirtes2000causation} recover the MEC but cannot orient edges outside V-structures; going beyond the MEC requires model assumptions that break this symmetry. Identifiability has been shown for linear Gaussian structural equation models (SEMs) with equal noise variances \cite{peters2014identifiability}, linear non-Gaussian SEMs \cite{shimizu2006linear}, and non-linear additive noise models \cite{peters2014causal}.

This prior work focuses on continuous valued random variables and additive noise models, a mold that real observational datasets rarely fit. Our current work is motivated by causal discovery in epidemiology, where ordinal, binary, and counting data appear alongside continuous-valued measurements \cite{olival2017host, johnson2020global, kreuder2015spillover}. Forcing these into a Gaussian additive-noise model introduces mis-specifications that invalidate both identifiability guarantees and downstream inferences.

A key contribution is the \emph{structured statistical model} (SSM), which generalizes structural equation models and appears to induce stronger identifiability properties. We leverage prior identifiability work for specific {\bf discrete} families: Poisson DAGs \cite{park2015learning}, extended to the generalized hypergeometric family in \cite{park2019identifiability}, and ordinal DAGs \cite{ni2022ordinal}. We generalize our prior work in \cite{shaska2025ordinal}, which considered mixed ordinal and Poisson nodes.

We also adapt key algorithms to the SSM. Since the DAG space grows super-exponentially with the number of nodes \cite{chickering2002optimal}, methods such as NOTEARS \cite{zheng2018dags}, GOLEM \cite{ng2020role}, DAGMA \cite{bello2022dagma}, and SURE-Ridge \cite{mishra2026causal} reformulate DAG learning as a continuous optimization problem with a smooth acyclicity characterization or regress-and-threshold style DAG learning, but focus on linear or additive noise models.

The contributions of this work are as follows:
\begin{enumerate}
\item We introduce the \emph{structured statistical model} (SSM), which generalizes structural equation models.
    \item We prove that the edge between an ordinal node and a regular one-parameter exponential family node in a bivariate SSM is distributionally identifiable for generic parameter values, subsuming \cite{shaska2025ordinal}.
    \item We develop a masked DAGMA procedure with joint cutpoint and weight optimization for graphs beyond the reach of exhaustive search.
    \item We validate the framework on three-node and $50$-node experiments, where the normalized structural Hamming distance (SHD) converges to zero with sample size, including correct orientation of edges within the MEC.
\end{enumerate}

\section{Background}
\label{sec:bg}

We consider a causal directed acyclic graph (DAG) with $d$ nodes, denoted by
$\mathcal{G} = \parn{\mathcal{V}, \mathcal{E}}$, where
$\mathcal{V} = \curl{1, \dots, d}$ is the vertex set and
$\mathcal{E} = \curl{ \parn{i,j} : i \rightarrow j }$ is the edge set.
The graph $\mathcal{G}$ is represented by a deterministic weighted adjacency
matrix $\mathbf{W}_{\mathcal{G}} \in \mathbb{R}^{d \times d}$ whose $\parn{i,j}$ entry is nonzero if and only if $\parn{i,j} \in \mathcal{E}$. Each node $i$ is associated with a random variable $X_i$ and parent set $\mathcal{P}\parn{i} = \curl{k : \parn{k,i} \in \mathcal{E}}$, and we write $\vect{x} = \parn{X_1, \dots, X_d}$.
For each node, the conditional distribution
$p\!\parn{ X_i \mid \mathbf{x}_{\mathcal{P}\parn{i}} }$
belongs to a known parametric family whose parameters are deterministic
functions of the parent variables. \emph{This generalizes the structural equation models typically seen in causal inference; we call it our \textbf{structured statistical model (SSM)}.}
The observed data matrix $\mathbf{X} \in \mathbb{R}^{n \times d}$ contains $n$ independent realizations of $\vect{x}$.

\subsection{Ordinal Nodes}
Let $i \in \mathcal{V}$ correspond to an ordinal random variable $X_i$ with finite ordered support $\mathcal{X}_i = \curl{1, \dots, s}$, $s \geq 2$.
If $\mathcal{P}\parn{i} = \emptyset$, $X_i$ is categorical with $p\parn{X_i = x} = \pi_{i,x}$. Otherwise, we adopt a cumulative-link (ordinal regression) model. Let $\vect{w}_i := \sqr{\vect{W}_{\mathcal{G}}}_{\mathcal{P}\parn{i},\, i}$ denote the causal coefficients from the parents of node $i$, with parent vector $\vect{x}_{\mathcal{P}\parn{i}}$, and introduce strictly ordered cutpoints $-\infty = \gamma_{i,0} < \gamma_{i,1} < \cdots < \gamma_{i,s-1} < \gamma_{i,s} = +\infty$.
The conditional probability mass function is
\begin{align}
p\parn{X_i = x \mid \vect{x}_{\mathcal{P}\parn{i}}}
&=
f\!\parn{
\gamma_{i,x} -
\vect{w}_i^{\mathsf T}\vect{x}_{\mathcal{P}\parn{i}}
}
\nonumber \\
&\quad {}
-
f\!\parn{
\gamma_{i,x-1} -
\vect{w}_i^{\mathsf T}\vect{x}_{\mathcal{P}\parn{i}}
}.
\end{align}
Throughout this work, we adopt the ordered logit model \cite{mccullagh1980regression}, $f\parn{u} = \sigma\parn{u} = 1/\parn{1 + e^{-u}}$; location shifts are absorbed into the cutpoints.

\subsection{Exponential-Family Nodes}
Let $i \in \mathcal{V}$ correspond to a random variable $X_i$ whose conditional
distribution belongs to a regular one-parameter exponential family with known
sufficient statistic $T_i:\mathcal{X}_i \to \mathbb{R}$ and natural parameter space $\Omega \subseteq \mathbb{R}$.
The support $\mathcal{X}_i \subseteq \mathbb{R}$ may be discrete or continuous. If $\mathcal{P}\parn{i} = \emptyset$, the natural parameter is a fixed $\eta_i \in \Omega$. Otherwise, we model it as a deterministic
function of a linear predictor:
\[
\vect{w}_i :=
\sqr{\vect{W}_{\mathcal{G}}}_{\mathcal{P}\parn{i},\, i},
\qquad
\eta_i\parn{\vect{x}_{\mathcal{P}\parn{i}}}
=
g_i\!\parn{
\vect{w}_i^{\mathsf T}\vect{x}_{\mathcal{P}\parn{i}}
},
\]
where $g_i:\mathbb{R} \to \Omega$ is a known, strictly monotone link function.
The conditional distribution becomes
\begin{align}
&p_{X_i \mid \vect{x}_{\mathcal{P}\parn{i}}}\parn{x}
= \nonumber \\
&H_i\parn{x}
\exp\!\parn{
\eta_i\parn{\vect{x}_{\mathcal{P}\parn{i}}} T_i\parn{x}
-
a_i\!\parn{
\eta_i\parn{\vect{x}_{\mathcal{P}\parn{i}}}
}
},
\end{align}
where $H_i\parn{x} > 0$ is the base measure and $a_i\parn{\cdot}$ is the log-partition function.

\section{Identifiability}
\label{sec:id}

\begin{table*}[t]
\centering
\caption{Regular One-Parameter Exponential Family Distributions}
\label{tab:exp_fam_dists}
\begin{tabular*}{\textwidth}{@{\extracolsep{\fill}} l l l l l@{}}
\hline
\textbf{Distribution} & \textbf{Natural Parameter} & \textbf{Link Function} & \textbf{Suff. Stat.} & \textbf{Support} \\
& $\eta$ & $\eta = g\parn{\vect{w}_i^{\mathsf T}\vect{x}_{\mathcal{P}\parn{i}}}$ & $T(x)$ & $\mathcal{X}$ \\
\hline
Exponential & $-\lambda, \quad (\lambda > 0)$ & $\eta = -e^{\mathbf{w}_i^\mathsf T\mathbf{x}_{\mathcal{P}(i)}}$ & $x$ & $\left[0, \infty\right)$ \\
Poisson & $\log \lambda, \quad (\lambda > 0)$ & $\eta = \mathbf{w}_i^\mathsf T\mathbf{x}_{\mathcal{P}(i)}$ & $x$ & $\{0, 1, \dots\}$  \\
Gaussian (fixed variance $\sigma^2$) & $\frac{\mu}{\sigma^2}, \quad (\mu \in \mathbb{R})$ & $\eta = \mathbf{w}_i^\mathsf T\mathbf{x}_{\mathcal{P}(i)}$ & $x$ & $\mathbb{R}$ \\
Gamma (fixed shape $\alpha$) & $-\beta, \quad (\beta > 0)$ & $\eta = -e^{\mathbf{w}_i^\mathsf T\mathbf{x}_{\mathcal{P}(i)}}$ & $x$ & $(0, \infty)$  \\
Binomial (fixed trials $m \geq 3$) & $\log\frac{p}{1 - p}, \quad (p \in (0,1))$ & $\eta = \mathbf{w}_i^\mathsf T\mathbf{x}_{\mathcal{P}(i)}$ & $x$ & $\{0, \dots, m\}$ \\
Pascal (fixed successes $r$) & $\log(1 - p), \quad (p \in (0,1))$ & $\eta = -\log(1 + e^{\mathbf{w}_i^\mathsf T\mathbf{x}_{\mathcal{P}(i)}})$ & $x$ & $\{0, 1, \dots\}$ \\
Gamma (fixed rate $\beta$) & $\alpha - 1$, $\parn{\alpha > 0}$ & $\eta = e^{\mathbf{w}_i^\mathsf T\mathbf{x}_{\mathcal{P}(i)}} - 1$ & $\log\parn{x}$ & $\parn{0, \infty}$ \\
\hline
\end{tabular*}
\end{table*}

We begin with a two-node system with random variables $X$ and $Y$.
We study two competing causal models $\mathcal{M}_{X \rightarrow Y} \text{ and }\mathcal{M}_{Y \rightarrow X}$, with the edge weight of the single directed edge between $X$ and $Y$ being $w \neq 0$. We assume that $X$ is an ordinal random variable with support $\mathcal{X} = \curl{1, \dots, s}$, and that the conditional distribution of $Y$ belongs to a regular one-parameter exponential family with known sufficient statistic $T\parn{y}$. Under $\mathcal{M}_{X \rightarrow Y}$, $X$ is a root node with probabilities $\pi_x$ and $Y \mid X = x$ has natural parameter $\eta\parn{x} = g\parn{wx}$, so that
\begin{equation}
p_{X,Y}\parn{x,y;\mathcal{M}_{X \rightarrow Y}}
=
\pi_x \, H\parn{y}
\exp\!\parn{
\eta\parn{x} T\parn{y} - a\parn{\eta\parn{x}}
}.
\label{eq:forward_joint}
\end{equation}
Under $\mathcal{M}_{Y \rightarrow X}$, $Y$ is a root node with natural parameter $\eta$ and $X \mid Y = y$ follows the ordered logit with cutpoints $\boldsymbol{\gamma}$ and linear predictor $wy$, so that
\begin{align}
&p_{X,Y}\parn{x,y;\mathcal{M}_{Y \rightarrow X}}
= \nonumber\\
& H\parn{y}
\exp\!\parn{
\eta T\parn{y} - a\parn{\eta}
}
\parn{
\sigma\parn{\gamma_x - wy}
-
\sigma\parn{\gamma_{x-1} - wy}
}.
\label{eq:reverse_joint}
\end{align}

We use the following assumptions throughout the work.
\begin{assumption}
\label{asm:all}
(i) $X$ has at least three ordinal levels, $s \geq 3$; (ii) the support of $Y$ contains at least three distinct points for discrete $Y$ and extends to $+\infty$ for continuous $Y$; (iii) the sufficient statistic $T\parn{y}$ is one of $y$, $\log\parn{y}$, or $y^k$ with fixed $k>0$.
\end{assumption}
Our key results, all under Assumption~\ref{asm:all}, are as follows.

\begin{lem}
\label{lem1}
For
$\mathcal{M}_{X \rightarrow Y}$ and any distinct $u,v \in \mathcal{X}$,
\[
\log\!\parn{\frac{p_{X \mid Y}\parn{u\mid y}}{p_{X \mid Y}\parn{v\mid y}}}
\]
is an affine function of $T\parn{y}$ on the support of $Y$.
\end{lem}
\begin{proof}[Proof sketch]
By Bayes' rule and \eqref{eq:forward_joint}, the base measure $H\parn{y}$ cancels and the log-ratio equals $\parn{\eta_u - \eta_v}\,T\parn{y} + \log\parn{\pi_u/\pi_v} - \parn{a\parn{\eta_u} - a\parn{\eta_v}}$, which is affine in $T\parn{y}$.
\end{proof}

\begin{lem}
Under $\mathcal{M}_{Y \rightarrow X}$ with the ordered logit cumulative-link model, for any distinct categories $u, v \in \mathcal{X}$, the posterior ratio
\[
R\parn{y; u, v} \triangleq \frac{p_{X \mid Y}\parn{u \mid y}}{p_{X \mid Y}\parn{v \mid y}}
\]
is a strictly positive rational function of $e^{wy}$ on the support of $Y$. Its leading-order behavior as $e^{wy} \to 0^{+}$ and $e^{wy} \to \infty$ depends only on whether $u$ or $v$ lies at the boundary of $\mathcal{X}$: if both are interior, $R$ tends to a positive constant at both ends; if exactly one is a boundary category, $R \sim c\, e^{\pm wy}$ at one end and tends to a positive constant at the other; if both are boundary, $R \sim c\, e^{\pm wy}$ at both ends. The sign of the exponent is determined by which side carries the boundary category, and all constants are strictly positive and finite, depending only on the cutpoints $\boldsymbol{\gamma}$.
\label{lem:myx_asymptote}
\end{lem}
\begin{proof}[Proof sketch]
With $t = e^{wy}$, the ordered logit gives $p_{X\mid Y}\parn{1\mid y} = e^{\gamma_1}/\parn{e^{\gamma_1}+t}$, $p_{X\mid Y}\parn{s\mid y} = t/\parn{e^{\gamma_{s-1}}+t}$, and $p_{X\mid Y}\parn{x\mid y} = \parn{e^{\gamma_x}-e^{\gamma_{x-1}}}\,t/\sqr{\parn{e^{\gamma_x}+t}\parn{e^{\gamma_{x-1}}+t}}$ for interior $x$; boundary and interior categories have different leading orders in $t$ at each end, and forming ratios yields the three cases.
\end{proof}

\begin{lem}
\label{lem2}
There do not exist parameters $\parn{\boldsymbol{\gamma}, w}$ for which
\[
\log\!\parn{\frac{p_{X \mid Y}\parn{u\mid y}}{p_{X \mid Y}\parn{v\mid y}}}
\]
is affine in $T\parn{y}$, on a continuous interval of $y$, for any distinct $u,v \in \mathcal{X}$. Furthermore, for discrete $Y$, this affine relationship cannot hold for generic parameters $\parn{\boldsymbol{\gamma}, w}$.
\end{lem}
\begin{proof}[Proof sketch]
Suppose $R\parn{y;u,v} = e^{\beta}e^{\alpha T\parn{y}}$ on the support; if $Y$ is continuous, real-analyticity extends this to a half-line. For $T\parn{y}=y$, this reads $R = e^{\beta}t^{r}$ with $r=\alpha/w$; matching leading-order exponents from Lemma~\ref{lem:myx_asymptote} at $t\to 0^+$ and $t\to\infty$ forces $r\in\curl{-1,0,+1}$. $r=0$ forces coinciding cutpoints, contradicting $u\neq v$; $r=\pm 1$ forces $\curl{u,v}=\curl{1,s}$ and then $e^{\gamma_1}=e^{\gamma_{s-1}}$, contradicting $s\ge 3$. For $T\parn{y}=\log y$ or $y^{k}$ ($k\neq 1$), $\log R$ would grow as $\Theta\parn{\log y}$ or $\Theta\parn{y^{k}}$, incompatible with the $O\parn{1}$ or $\Theta\parn{y}$ growth of $\log R$ under Lemma~\ref{lem:myx_asymptote}, forcing $\alpha=0$ and reducing to the constant case. For discrete $Y$, three support points yield a collinearity condition $\Phi\parn{\boldsymbol{\gamma},w}=0$ with $\Phi$ real-analytic and not identically zero, so it fails outside a Lebesgue-null set.
\end{proof}

\begin{thm}
\label{thm1}
The models $\mathcal{M}_{X \rightarrow Y}$ and
$\mathcal{M}_{Y \rightarrow X}$ are distributionally identifiable (generically identifiable when $Y$ is discrete).
\end{thm}
\begin{proof}[Proof sketch]
Suppose forward parameters $\parn{\boldsymbol{\eta},\boldsymbol{\pi}}$ and reverse parameters $\parn{\eta,\boldsymbol{\gamma},w}$ yield the same joint, hence the same posteriors $p_{X\mid Y}\parn{x\mid y}$. By Lemma~\ref{lem1}, every log-posterior ratio is affine in $T\parn{y}$, which Lemma~\ref{lem2} rules out for $w\neq 0$ (generically, for discrete $Y$). Hence no such reverse parameters exist.
\end{proof}

\section{Algorithms for Mixed DAG Discovery}
\label{sec:alg}

\begin{algorithm}[!htbp]
\caption{Masked DAGMA with m-NLL Score for Large DAG Discovery}
\label{alg:dagma_nll}
\begin{algorithmic}[1]
\Require Data matrix $\vect{X}$, number of nodes $d \geq 6$,
         bipartite partition $(\mathcal{V}_{\operatorname{ord}}, \mathcal{V}_{\operatorname{exp}})$,
         initial central path coefficient $\mu^{(0)}$ (e.g., 1),
         decay factor $\alpha \in (0,1)$ (e.g., 0.1),
         $\ell_1$ parameter $\lambda > 0$ (e.g., 0.01),
         log-det parameter $\tau > 0$ (e.g., 1),
         number of iterations $T$,
         threshold $\omega > 0$ (e.g., 0.3)
\Ensure Estimated DAG $\mathcal{G}_{\operatorname{est}}^{*}$ induced by 
        $\vect{W}_{\operatorname{est}}^{*}$

\State Construct bipartite mask $\vect{M} \in \{0,1\}^{d \times d}$ where 
       $[\vect{M}]_{i,j} = 0$ if nodes $i$ and $j$ belong to the same type,
       and $[\vect{M}]_{i,j} = 1$ otherwise

\State Initialize $\vect{W}_{\operatorname{est}}^{(0)} = \mathbf{0}_{d \times d}$ and 
       $\curl{\boldsymbol{\gamma}_i^{(0)}}_{i \in \mathcal{V}_{\operatorname{ord}}} = \vect{0}$ 

\For{$t = 0, 1, 2, \ldots T-1$}
    \State Starting at $\vect{W}_{\operatorname{est}}^{(t)}$ and 
           $\curl{\boldsymbol{\gamma}_i^{(t)}}_{i \in \mathcal{V}_{\operatorname{ord}}}$, solve
    \begin{eqnarray}
        &&\!\!\!\!\!\!\!\!\!\!\!\!\!\!\!\!\!\parn{\vect{W}_{\operatorname{est}}^{(t+1)}, \curl{\boldsymbol{\gamma}_i^{(t+1)}}} 
        = \nonumber \\ &&\!\!\!\!\!\!\!\!\!\!\!\arg\min_{\vect{W}, \curl{\boldsymbol{\gamma}_i}} \ \mu^{(t)} 
        Q\parn{\vect{W}, \curl{\boldsymbol{\gamma}_i}; \vect{X}}
        + h\parn{\vect{W} \odot \vect{M}} \nonumber\\
        &&\!\!\!\!\!\!\!\!\!\!\!\text{where} \quad Q\parn{\vect{W}, \curl{\boldsymbol{\gamma}_i}; \vect{X}}\nonumber \\
        && \qquad = 
        e\parn{\vect{W} \odot \vect{M}, \curl{\boldsymbol{\gamma}_i}_{i \in \mathcal{V}_{\operatorname{ord}}}; \vect{X}}
         + \lambda\|\vect{W} \odot \vect{M}\|_1 \nonumber
    \end{eqnarray}
    \State Set $\mu^{(t+1)} = \alpha \mu^{(t)}$
\EndFor

\State Apply bipartite mask to final iterate: 
       $\vect{W}_{\operatorname{est}}^{(T)} \gets \vect{W}_{\operatorname{est}}^{(T)} \odot \vect{M}$

\State Threshold small entries: 
       $[\vect{W}_{\operatorname{est}}^{*}]_{i,j} \gets 
       [\vect{W}_{\operatorname{est}}^{(T)}]_{i,j} \cdot 
       \mathbf{1}\sqr{|[\vect{W}_{\operatorname{est}}^{(T)}]_{i,j}| \geq \omega}$

\State \Return $\mathcal{G}_{\operatorname{est}}^{*}$

\end{algorithmic}
\end{algorithm}

Let $\mathcal{V} = \mathcal{V}_{\operatorname{ord}} \cup \mathcal{V}_{\operatorname{exp}}$ be partitioned into ordinal and exponential family nodes, with edges permitted only between nodes of different types, so $\mathcal{G}$ is bipartite; we encode this with the mask $\vect{M} \in \curl{0,1}^{d \times d}$, where $\sqr{\vect{M}}_{i,j} = 0$ if $i$ and $j$ are of the same type and $1$ otherwise. For a candidate graph $\mathcal{G}_{\operatorname{est}}$ with weighted adjacency $\vect{W}_{\operatorname{est}}$ and parent sets $\mathcal{P}_{\mathcal{G}_{\operatorname{est}}}\parn{i}$, we score the fit by the mean negative log-likelihood (m-NLL)
\begin{IEEEeqnarray}{rCl}
\IEEEeqnarraymulticol{3}{l}{e\parn{\vect{W}_{\operatorname{est}}, \curl{\boldsymbol{\gamma}_i}_{i \in \mathcal{V}_{\operatorname{ord}}}; \vect{X}}} \nonumber \\
\quad &\triangleq& -\frac{1}{n}\sum_{j = 1}^{n}\sum_{i = 1}^{d}\ln p\parn{\vect{X}_{j,i} \mid \vect{X}_{j, \mathcal{P}_{\mathcal{G}_{\operatorname{est}}}\parn{i}}; \vect{W}_{\operatorname{est}}, \boldsymbol{\gamma}_i},
\label{eq:err_fn}
\end{IEEEeqnarray}
where $\boldsymbol{\gamma}_i$ are the cutpoints of ordinal node $i$; \eqref{eq:err_fn} decomposes over nodes. We use two strategies depending on $d$.

\subsection{Exhaustive search for small graphs} For $d \leq 5$ we enumerate the feasible set $\mathcal{U}'\parn{d}$ of all acyclic binary matrices $\vect{W}_{\operatorname{est},B} \in \curl{0,1}^{d \times d}$ satisfying $\vect{W}_{\operatorname{est},B} \odot \parn{\mathbf{1} - \vect{M}} = \mathbf{0}$, which always contains the true graph, and solve
\begin{equation}
\mathcal{G}_{\operatorname{est}}^{*} = \arg\min_{\mathcal{G}_{\operatorname{est}} \in \mathcal{U}'\parn{d}} e\parn{\vect{W}_{\operatorname{est}}, \curl{\boldsymbol{\gamma}_i}_{i \in \mathcal{V}_{\operatorname{ord}}}; \vect{X}}.
\label{eq:es_opt}
\end{equation}
For each candidate skeleton, the edge weights of exponential family nodes and the joint weights and cutpoints of ordinal nodes are obtained by per-node conditional maximum likelihood via L-BFGS-B, with the cutpoints reparametrized through softplus increments to enforce strict ordering. Exhaustive search returns the exact minimizer of \eqref{eq:es_opt}, so any recovery error is attributable to finite-sample noise alone.

\subsection{Masked continuous optimization for large graphs} Since $\bars{\mathcal{U}'\parn{d}}$ grows super-exponentially in $d$, for $d \geq 6$ we embed structure learning into a continuous optimization over weighted adjacency matrices using DAGMA \cite{bello2022dagma} as the backbone. The bipartite constraint is enforced throughout by optimizing over the masked parameters $\vect{W} \odot \vect{M}$, since the $\ell_1$ penalty alone does not guarantee a bipartite solution:
\begin{IEEEeqnarray}{rCl}
\vect{W}_{\operatorname{est}}^{*} &=& \arg\min_{\vect{W} \in \mathbb{R}^{d \times d},\, \curl{\boldsymbol{\gamma}_i}} e\parn{\vect{W} \odot \vect{M}, \curl{\boldsymbol{\gamma}_i}; \vect{X}} + \lambda\norm{\vect{W} \odot \vect{M}}_1 \nonumber \\
&& \text{subject to} \quad h\parn{\vect{W} \odot \vect{M}} = 0,
\label{eq:dagma_opt}
\end{IEEEeqnarray}
where $\lambda > 0$ controls sparsity and $h\parn{\vect{W}} = -\log\det\parn{\tau\vect{I}_d - \vect{W} \circ \vect{W}} + d\log\tau$, $\tau > 0$, vanishes if and only if $\vect{W}$ is acyclic. The cutpoints are optimized jointly with $\vect{W}$, since the skeleton is itself a product of the optimization. Denoting the penalized objective by $Q\parn{\vect{W}, \curl{\boldsymbol{\gamma}_i}; \vect{X}} \triangleq e\parn{\vect{W} \odot \vect{M}, \curl{\boldsymbol{\gamma}_i}; \vect{X}} + \lambda\norm{\vect{W} \odot \vect{M}}_1$, we solve \eqref{eq:dagma_opt} by the central path method of \cite{bello2022dagma}, with the mask applied before every loss, gradient, and acyclicity evaluation, as summarized in Algorithm~\ref{alg:dagma_nll}.

\begin{figure*}[t]
    \centering
    \subfigure[$\mathcal{G}_1: X_1 \rightarrow X_2 \rightarrow X_3$]{
        \includegraphics[width=0.42\textwidth]{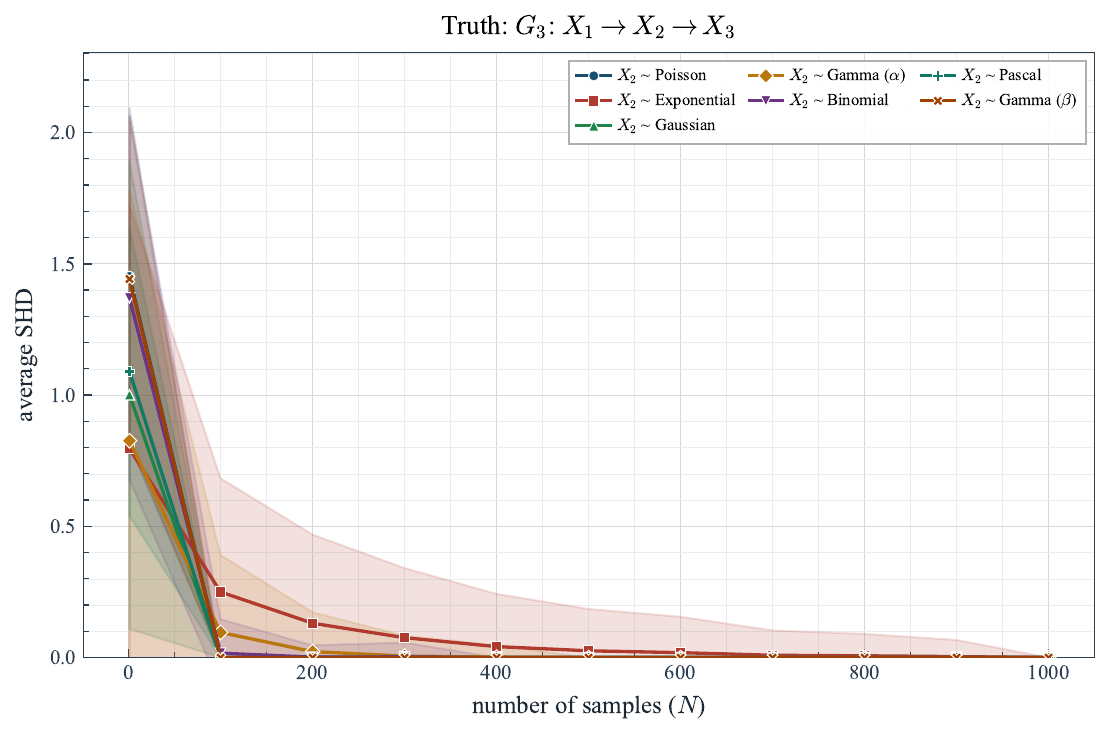}
        \label{fig:g1_shd}
    }
    \hfill
    \subfigure[$\mathcal{G}_2: X_1 \rightarrow X_2 \leftarrow X_3$]{
        \includegraphics[width=0.42\textwidth]{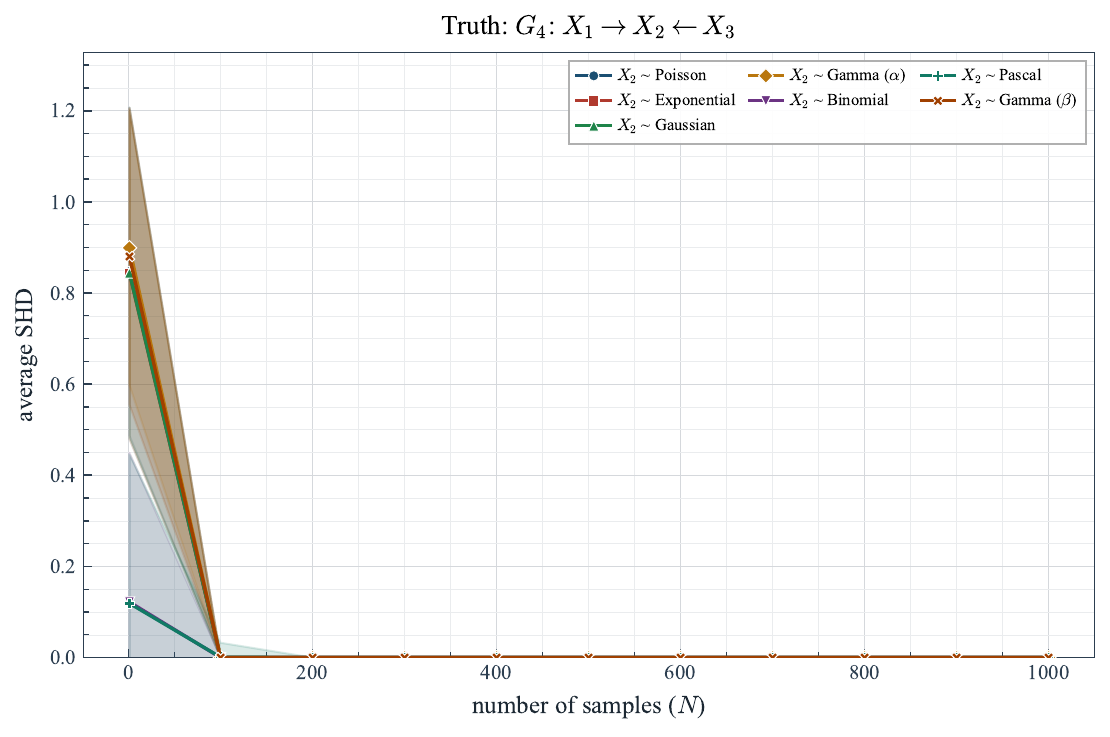}
        \label{fig:g2_shd}
    }

    \subfigure[$\mathcal{G}_3: X_1 \leftarrow X_2 \rightarrow X_3$]{
        \includegraphics[width=0.42\textwidth]{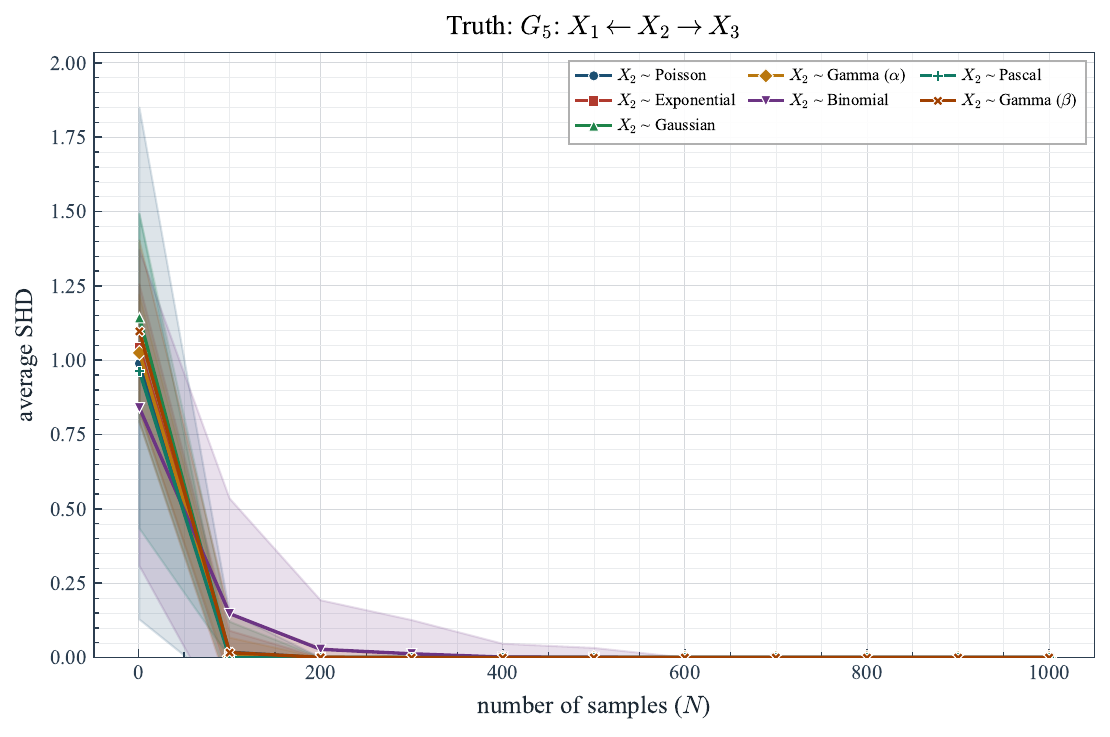}
        \label{fig:g3_shd}
    }
    \hfill
    \subfigure[$\mathcal{G}_4: X_1 \leftarrow X_2 \leftarrow X_3$]{
        \includegraphics[width=0.42\textwidth]{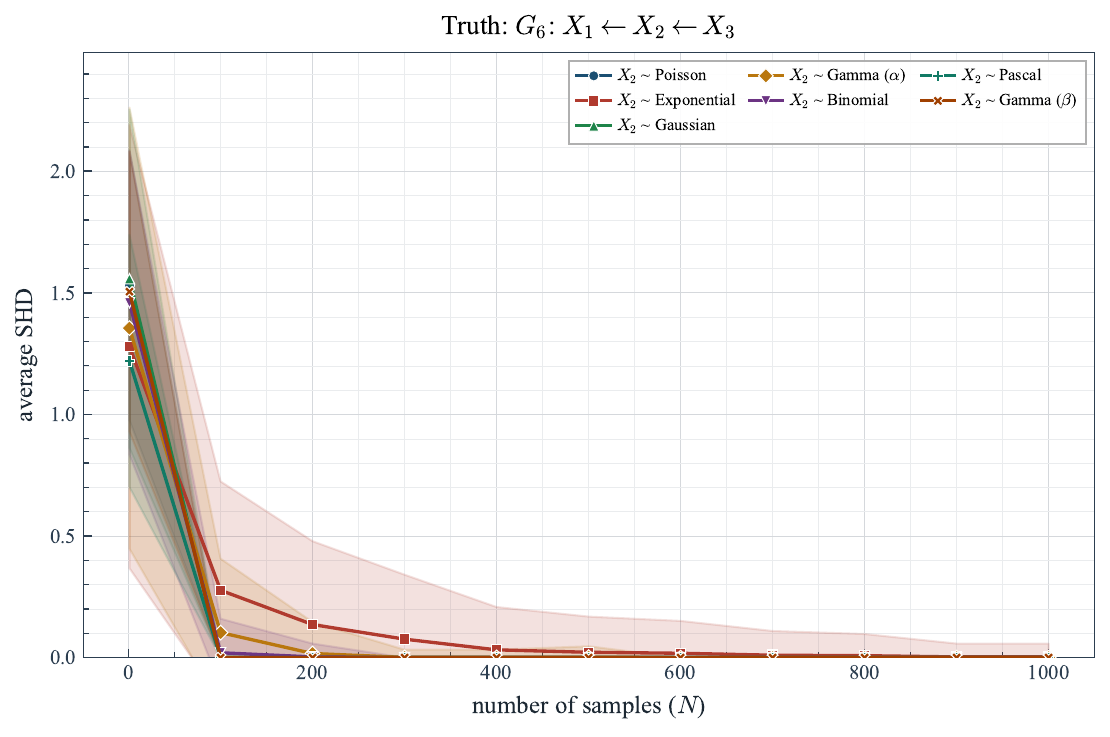}
        \label{fig:g4_shd}
    }
    \caption{Normalized SHD versus number of samples for the four three-node graphs with edge weights in $\curl{0, 1}$.}
    \label{fig:three_node_shd}
\end{figure*}

\begin{figure*}[t]
    \centering
    \subfigure[ER-2]{
        \includegraphics[width=0.48\textwidth]{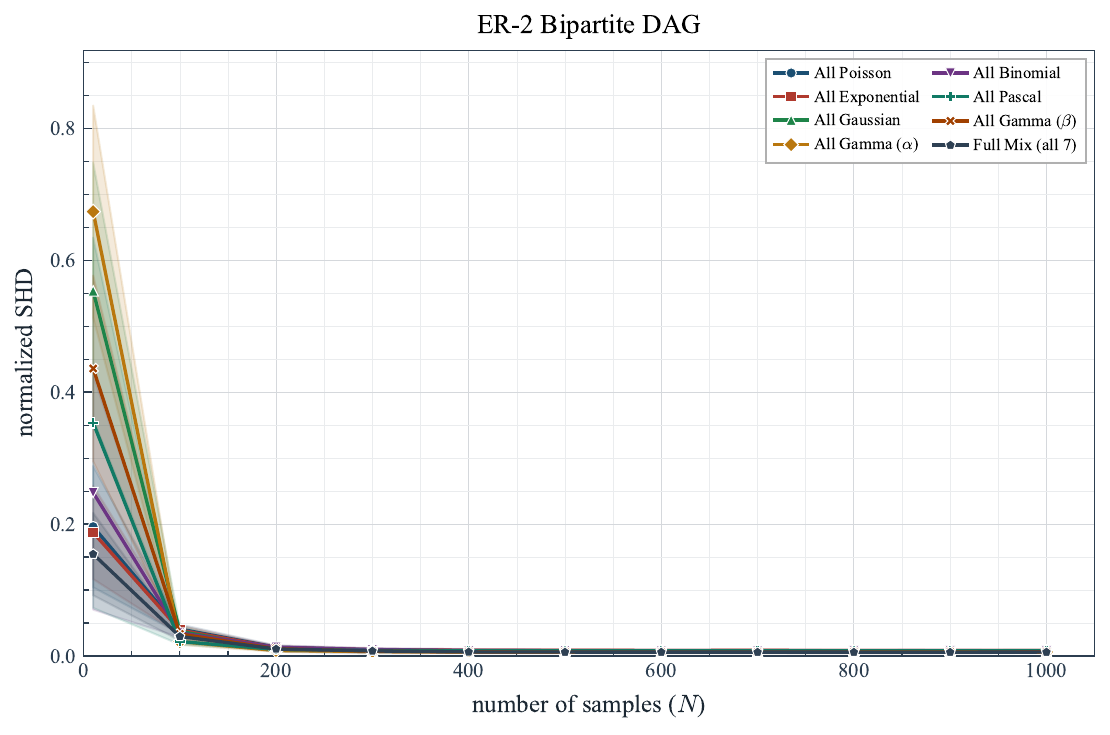}
        \label{fig:d50_er2_nshd}
    }
    \hfill
    \subfigure[ER-4]{
        \includegraphics[width=0.48\textwidth]{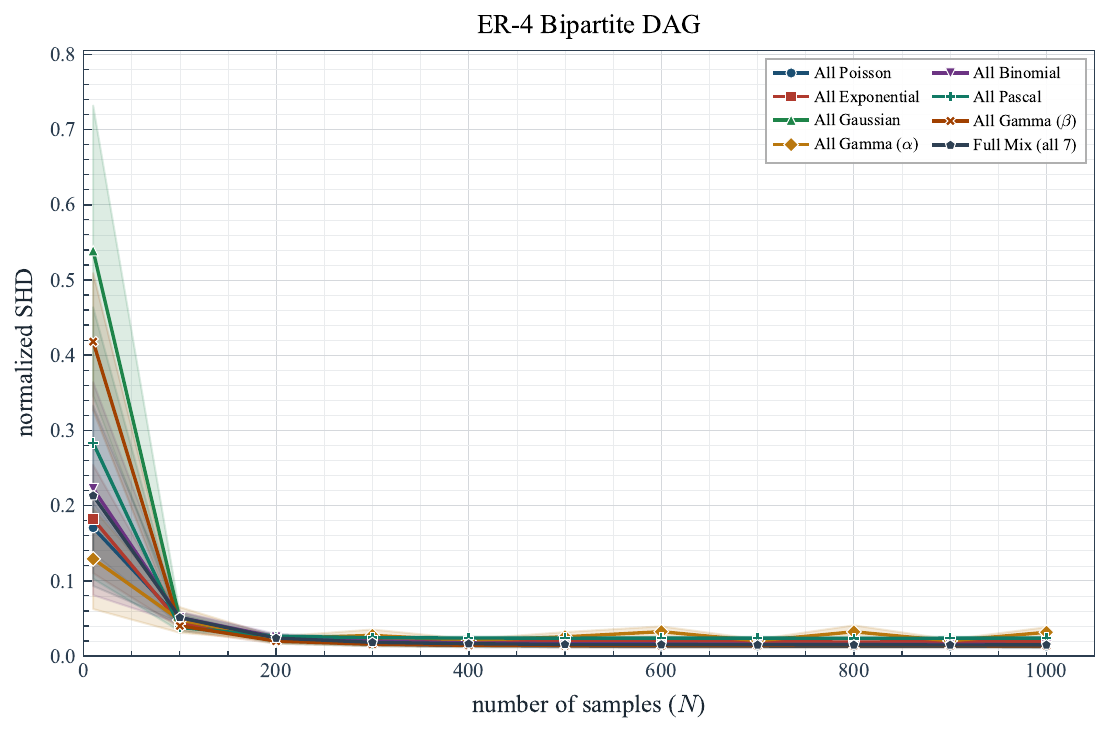}
        \label{fig:d50_er4_nshd}
    }
    \caption{Normalized SHD versus number of samples for $d = 50$ node bipartite DAGs recovered by masked DAGMA.}
    \label{fig:d50_dagma}
\end{figure*}

\section{Numerical Results and Discussion}
\label{sec:num}

We validate the identifiability result using the algorithms of Section~\ref{sec:alg}: exhaustive search on the $d = 3$ case, the smallest setting containing distinct Markov equivalence classes, and masked DAGMA on $d = 50$ bipartite DAGs. Classically, graphs within a single MEC are not identifiable; Theorem~\ref{thm1} predicts that the SSM breaks this symmetry for ordinal-exponential family edges. We measure recovery by the structural Hamming distance (SHD), which counts all errors in the presence/absence and direction of edges, normalized by the number of possible directed edges (nSHD). We consider four ground truth DAGs:
\begin{align*}
    \mathcal{G}_1&: X_1 \rightarrow X_2 \rightarrow X_3, &
    \mathcal{G}_2&: X_1 \rightarrow X_2 \leftarrow X_3, \\
    \mathcal{G}_3&: X_1 \leftarrow X_2 \rightarrow X_3, &
    \mathcal{G}_4&: X_1 \leftarrow X_2 \leftarrow X_3.
\end{align*}
$\mathcal{G}_1$, $\mathcal{G}_3$, and $\mathcal{G}_4$ share a MEC and are classically unidentifiable from observations alone, whereas Theorem~\ref{thm1} predicts that all four are distinguishable under the SSM. Here $X_2$ is the exponential family node, drawn from each of the seven distributions in Table~\ref{tab:exp_fam_dists} in turn, and $X_1$, $X_3$ are ordinal with $s = 4$ categories.

Figure \ref{fig:three_node_shd} shows nSHD vs the number of samples for these four graphs, with one curve per distribution of $X_2$, all edge weights in $\curl{0, 1}$, and $N$ ranging from $1$ to $1000$. For all $4$ DAGs and all seven distributions, the nSHD approaches $0$ as the number of samples increases, empirically demonstrating the identifiability of edge directions between ordinal and exponential-family nodes. Figure \ref{fig:d50_dagma} presents the results for the $50$ node bipartite DAG case, which uses the masked DAGMA algorithm, where ER-$k$ denotes Erd\H{o}s--R\'enyi graphs with $kd$ expected edges. We consider eight distribution modes: seven homogeneous modes in which all exponential family nodes share one distribution from Table~\ref{tab:exp_fam_dists}, and a mixed mode in which each exponential family node is assigned one of the seven uniformly at random. Across both ER-2 and ER-4 settings, the nSHD decreases sharply with increasing $N$, approaching zero by moderate sample sizes for most modes, confirming that the masked DAGMA estimator recovers both the sparsity pattern and the edge orientations as the data grows.

\section{Conclusions}
\label{sec:conc}
We establish identifiability for causal discovery in DAGs whose nodes follow either ordinal or regular one-parameter exponential family distributions: the direction of the edge between an ordinal node and an exponential family node in a bivariate SSM is uniquely determined from the joint distribution for generic parameter values, subsuming the Ordinal-Poisson result of \cite{shaska2025ordinal}. We develop an exhaustive search for small DAGs and a masked DAGMA optimization for larger DAGs, and our numerical results confirm the decay of nSHD to zero with sample size for three-node and $50$-node bipartite DAGs, signaling recovery of the true DAG even within the MEC.

\bibliographystyle{IEEEtran}
\bibliography{references}

\end{document}